\documentclass[10pt,twocolumn]{article}
\usepackage[margin=0.75in]{geometry}
\usepackage{amsmath,amssymb,amsthm,bm,mathtools,graphicx,enumitem,booktabs}
\usepackage{framed}
\usepackage[colorlinks=true,linkcolor=blue,citecolor=blue]{hyperref}

\newtheorem{theorem}{Theorem}
\newtheorem{lemma}{Lemma}
\newtheorem{proposition}{Proposition}
\newtheorem{corollary}{Corollary}
\newtheorem{assumption}{Assumption}
\newtheorem{definition}{Definition}
\newtheorem{remark}{Remark}

\newcommand{\R}{\mathbb{R}}
\newcommand{\E}{\mathbb{E}}
\newcommand{\Prob}{\mathbb{P}}
\newcommand{\norm}[1]{\left\|#1\right\|}
\newcommand{\inner}[2]{\langle #1, #2 \rangle}
\newcommand{\bw}{\bm{w}}
\newcommand{\bz}{\bm{z}}
\newcommand{\bx}{\bm{x}}
\newcommand{\by}{\bm{y}}
\newcommand{\ba}{\bm{a}}
\newcommand{\bb}{\bm{\beta}}
\newcommand{\bbs}{\bm{\beta}^\ast}
\newcommand{\bbt}{\tilde{\bm{\beta}}}
\newcommand{\bDelta}{\bm{\Delta}}
\newcommand{\bA}{\bm{A}}
\newcommand{\bB}{\bm{B}}
\newcommand{\bPhi}{\bm{\Phi}}
\newcommand{\bPsi}{\bm{\Psi}}
\newcommand{\dH}{\delta_{H}}
\newcommand{\psidH}{\psi_{\dH}}
\newcommand{\prox}{\mathrm{prox}}

\DeclareMathOperator{\supp}{supp}

\title{Demixing Sparse Signals from Nonlinear Observations using
Generalized Non-convex Regularization}
\author{Raziyeh Takbiri%
\thanks{R.~Takbiri is with the School of Electrical Engineering, Iran
University of Science and Technology, Tehran, Iran (e-mail:
raziyeh.takbiri@gmail.com). Code to reproduce all experiments is
available with this manuscript.}}
\date{}

\begin{document}
\maketitle

\begin{abstract}
We consider the recovery of a pair of sparse vectors from a limited
number of nonlinear observations of their superposition:
$y_i=g(\inner{\ba_i}{\bPhi\bw^\ast+\bPsi\bz^\ast})+e_i$, $i=1,\dots,m$,
with $m\ll n$, incoherent orthonormal bases $\bPhi,\bPsi$, a scalar link
$g$, and noise $e_i$ that may be heavy-tailed or contaminated. We
propose a regularization-based framework combining a Huberized data
fidelity with generalized folded-concave penalties (SCAD, MCP), and a
two-block proximal alternating algorithm with backtracking (NLD-PALM)
whose whole iterate sequence provably converges to critical points under
the Kurdyka--\L{}ojasiewicz property, with local linear rates. On the
statistical side we establish restricted strong convexity of the
Huberized nonlinear loss through an exact sign-definite decomposition,
and derive estimation error bounds of order
$\sigma\sqrt{s\log(n)/m}$ that hold at \emph{every} localized stationary
point, an oracle rate $\sigma\sqrt{s/m}$ free of $\log n$ and shrinkage
bias under a beta-min condition, and a co-equal recovery theorem for
\emph{unknown} monotone links via a linear surrogate and a clipped
Plan--Vershynin decoupling. The estimator requires no knowledge of the
sparsity levels, and its guarantees hold under symmetric noise with only
finite variance. Experiments at $n=512$ under a frozen data-driven
regularization rule show an earlier phase transition than convex
$\ell_1$ demixing and greedy hard-thresholding baselines, a $35\times$
accuracy advantage over squared-loss estimation under $5\%$ gross
outliers, and successful demixing of spike-plus-background signals
observed through a saturating amplifier.
\end{abstract}


\section{Introduction}\label{sec:intro}

Signal demixing --- separating a superposition of structured components
from few measurements --- and nonlinear compressed sensing have each
been studied extensively, but their intersection remains sparse.
Acquisition front ends saturate, clip, compand, and quantize; the
recorded data are then \emph{nonlinear} functions of a mixture. Formally
we observe
\begin{equation}\label{eq:model}
y_i = g\bigl(\inner{\ba_i}{\bx^\ast}\bigr)+e_i,\qquad
\bx^\ast=\bPhi\bw^\ast+\bPsi\bz^\ast,
\end{equation}
$i=1,\dots,m\ll n$, with $\norm{\bw^\ast}_0\le s_1$,
$\norm{\bz^\ast}_0\le s_2$, and ask for both components.

The state of the art for \eqref{eq:model} is the greedy line of Soltani
and Hegde \cite{soltani2017fast,soltani2017demixing}: demixing hard
thresholding with sample complexity $m=\mathcal O(s\log(n/s))$,
$s=s_1+s_2$, requiring the exact sparsity levels as input and offering
no robustness theory. In the \emph{linear} regime ($g=\mathrm{id}$),
non-convex regularization is known to outperform convex $\ell_1$
demixing \cite{wen2019efficient,wen2017survey}, reflecting the general
folded-concave phenomenon \cite{fan2001scad,zhang2010mcp}: penalties
that debias large coefficients achieve oracle behavior where the LASSO
carries an irreducible $\lambda\sqrt s$ shrinkage bias
\cite{loh2015regularized,loh2017support}. For single-index models
without superposition, Plan and Vershynin \cite{plan2016generalized}
showed that unknown monotone nonlinearities can be absorbed into an
effective scaling plus noise. To the best of our knowledge, no prior
work combines generalized non-convex regularization with nonlinear
demixing, in either the known- or unknown-link regime.

\subsection{Contributions}
\begin{enumerate}[leftmargin=1.5em]
\item \textbf{Estimator (Sec.~\ref{sec:formulation}).} A Huberized
non-convex program for \eqref{eq:model}. Huberization is structural,
not cosmetic: it eliminates a curvature--noise compatibility condition
that squared-loss analysis provably requires, and extends all guarantees
to symmetric noise with only finite variance.
\item \textbf{Algorithm (Sec.~\ref{sec:algorithm}).} NLD-PALM, a
two-block proximal-gradient method with per-block backtracking and an
over-relaxation factor $\eta>1$ that we show is \emph{necessary} for
sufficient decrease with non-convex penalties. Whole-sequence
convergence to critical points under the KL property
(Theorem~\ref{thm:conv}), with R-linear local rates; backtracking admits
links with unbounded derivative that break fixed-step schemes.
\item \textbf{Statistical theory (Sec.~\ref{sec:theory}).} Restricted
strong convexity of the Huberized nonlinear loss
(Theorem~\ref{thm:rsc}) via an exact decomposition whose curvature part
is sign-definite \emph{for every sample}; error bounds at all localized
stationary points (Theorem~\ref{thm:main}); an oracle property
(Proposition~\ref{prop:oracle}); and a co-equal unknown-link theorem
(Theorem~\ref{thm:unknown}) built on a clipped decoupling lemma
(Lemma~\ref{lem:decouple}).
\item \textbf{Experiments (Sec.~\ref{sec:experiments}).} Full-protocol
study at $n=512$ with a frozen, data-driven $\lambda$ rule: earlier
phase transitions than $\ell_1$ and DHT (the latter given the true
sparsity levels), robustness to gross outliers and heavy tails, and a
saturation-demixing application.
\end{enumerate}

All proofs are given in the main text or Appendix~\ref{app:uniform};
the companion conference paper \cite{takbiri2027icassp} announces a
subset of these results.

\section{Problem Formulation}\label{sec:formulation}

\subsection{Model and assumptions}

Write $\bB=[\bPhi\ \bPsi]\in\R^{n\times2n}$, $\bb=(\bw,\bz)\in\R^{2n}$,
truth $\bbs$, $S=\supp(\bbs)$, $s=|S|$.

\begin{assumption}[Sensing]\label{as:sensing}
$\ba_i\overset{\mathrm{iid}}\sim\mathcal N(\bm0,\bm I_n)$.
\end{assumption}

\begin{assumption}[Known link; primary regime]\label{as:link}
$g$ is known, differentiable, with $0<\ell_g\le g'(t)\le L_g$ and
$|g''(t)|\le M_g$ for $|t|\le\tau$, where $\tau$ satisfies
$\Prob\{\max_i|\inner{\ba_i}{\bx}|>\tau\}\le(mn)^{-1}$ for the signals
in play (Gaussian maximal inequality; see
Assumption~\ref{as:local}).
\end{assumption}

\begin{assumption}[Incoherence]\label{as:incoh}
$\varepsilon\coloneqq\max_{i,j}|\inner{\bm\phi_i}{\bm\psi_j}|$ satisfies
$16\,\varepsilon s\le\tfrac12$. For $\bPhi=\bm I$, $\bPsi=$ DCT,
$\varepsilon=\sqrt{2/n}$.
\end{assumption}

\begin{assumption}[Noise]\label{as:noise}
$e_i$ iid, independent of $\{\ba_i\}$, symmetric, and either (a)
$\sigma^2$-sub-Gaussian, or (b) of finite variance $\E e_i^2\le\sigma^2$
(heavy tails / $\epsilon$-contamination with symmetric contaminant).
\end{assumption}

\begin{assumption}[Localization]\label{as:local}
Statements about stationary points concern $\bbt$ with
$\norm{\bbt-\bbs}_2\le r$ where $(\norm{\bx^\ast}_2+2r)
\sqrt{2\log(2mn)}\le\tau$; boundedness of iterates is enforced by the
ball device of Remark~\ref{rem:bounded}.
\end{assumption}

\subsection{Penalty family}

\begin{definition}[Folded-concave class
$\mathcal P(\lambda,a)$]\label{def:pen}
$P_\lambda:\R\to\R_{\ge0}$ with: (P1) symmetric, $P_\lambda(0)=0$,
nondecreasing on $[0,\infty)$; (P2) differentiable on $(0,\infty)$ with
$0\le P'_\lambda\le\lambda$, $P'_\lambda$ nonincreasing; (P3)
$P'_\lambda(0^+)=\lambda$; (P4) $P'_\lambda(t)=0$ for $t\ge a\lambda$.
Vector penalties act coordinatewise.
\end{definition}

\begin{assumption}[$\rho$-amenability]\label{as:amenable}
$P_\lambda\in\mathcal P(\lambda,a)$ and
$\lambda|t|-\tfrac{\rho}{2}t^2\le P_\lambda(t)\le\lambda|t|$ with
$t\mapsto P_\lambda(t)+\tfrac{\rho}{2}t^2$ convex. MCP
($\gamma>1$): $\rho=1/\gamma$; SCAD ($a>2$): $\rho=1/(a-1)$. The
$\ell_q$ quasi-norms ($0<q<1$) violate (P3)--(P4) and weak convexity;
they are covered by the algorithmic theory (KL property) and
experiments, not by the statistical theorems --- a deliberate two-tier
structure defining what ``generalized'' means in our title.
\end{assumption}

\subsection{The Huberized estimator}

Let $h_{\dH}$ denote the Huber function with threshold $\dH$ and
$\psidH(t)=\max(-\dH,\min(t,\dH))$. The estimator is any localized
stationary point of
\begin{equation}\label{eq:objective}
F(\bb)=\underbrace{\frac1m\sum_{i=1}^m h_{\dH}\bigl(y_i-
g(\inner{\ba_i}{\bB\bb})\bigr)}_{f_H(\bb)}
+P_{\lambda}(\bw)+P_{\lambda}(\bz),
\end{equation}
with gradient $\nabla f_H(\bb)=-\frac1m\bB^\top\bA^\top[g'(\bm u)\odot
\psidH(\by-g(\bm u))]$, $\bm u=\bA\bB\bb$. Default $\dH=4\sigma$
(known link); the unknown-link variant uses an adaptive level
(Sec.~\ref{ssec:unknown}).

\begin{remark}[Why Huberize]\label{rem:whyhuber}
Squared-loss curvature analysis for \eqref{eq:model} meets the term
$\frac1m\sum_ie_i(g'(u_i)-g'(u_i^\ast))\inner{\ba_i}{\bB\bDelta}$, whose
natural bound carries $\E|e|\asymp\sigma$ --- a \emph{non-vanishing}
bias forcing a condition of the form $\sigma M_g\lesssim\ell_g^2$.
Huberization replaces $e_i$ by the bounded, mean-zero (under symmetry)
score $\psidH(e_i)$, whose empirical process concentrates; the condition
disappears, and only the variance of the noise enters the rates.
\end{remark}

\section{Algorithm and Convergence}\label{sec:algorithm}

\begin{framed}
\noindent\textbf{Algorithm NLD-PALM.} Parameters $L_{\min}>0$,
$\kappa>1$, $\delta\in(0,1)$, $\eta>1$. For $k=0,1,2,\dots$:
\begin{enumerate}[leftmargin=2em,itemsep=0pt]
\item ($w$-block) backtrack $L\ge L_{\min}$ (multiply by $\kappa$) until
$\bw^{+}=\prox_{P_\lambda/(\eta L)}\bigl(\bw^k-\tfrac{1}{\eta
L}\nabla_{\bw}f_H(\bw^k,\bz^k)\bigr)$ satisfies
$f_H(\bw^{+},\bz^k)\le f_H(\bw^k,\bz^k)
+\inner{\nabla_{\bw}f_H}{\bw^{+}-\bw^k}
+\tfrac{L}{2}\norm{\bw^{+}-\bw^k}^2$;
set $\bw^{k+1}=\bw^{+}$, $L_w^{k+1}=\delta L$.
\item ($z$-block) same with roles swapped, using $\bw^{k+1}$.
\end{enumerate}
Prox maps of MCP, SCAD, $\ell_{1/2}$ are closed-form
(Appendix~\ref{app:prox}).
\end{framed}

\begin{remark}[Necessity of $\eta>1$]\label{rem:eta}
The prox inequality at modulus $\eta L$ and the accepted descent
condition at constant $L$ add to a per-block decrease of
$\frac{(\eta-1)L}{2}\norm{\bDelta}^2$. At $\eta=1$ the quadratic terms
cancel exactly and only monotonicity survives: for \emph{convex}
penalties the prox inequality self-improves and $\eta=1$ suffices, but
for our non-convex penalties the sufficient-decrease property --- hence
the entire KL mechanism --- genuinely requires $\eta>1$.
\end{remark}

\begin{remark}[Boundedness device]\label{rem:bounded}
$F$ need not be coercive (MCP/SCAD are bounded; $\tanh$-type losses are
bounded). We add the indicator of a ball
$B_R=\{\norm{(\bw,\bz)}_2\le R\}$, $R$ large; in every experiment the
constraint is verified inactive a posteriori. For $\ell_q$, $F$ is
coercive and no device is needed.
\end{remark}

\begin{theorem}[Convergence of NLD-PALM]\label{thm:conv}
Suppose $g\in C^2$ (squared loss) or $f_H\in C^{1,1}_{\mathrm{loc}}$
(Huberized: $\nabla f_H$ is a composition of the $1$-Lipschitz clipping
with locally Lipschitz maps), the penalties are proper lsc prox-bounded,
$F$ is definable in an o-minimal structure (true for
polynomial/tanh/sigmoid links with MCP/SCAD/$\ell_q$: semialgebraic and
$\R_{\mathrm{an,exp}}$-definable pieces), and the iterates are bounded.
Then:
(i) backtracking terminates finitely at every step with accepted
constants in $[L_{\min},\bar L]$;
(ii) $F(\bw^k,\bz^k)$ decreases by at least
$\frac{(\eta-1)L_{\min}}{2}\norm{(\bw,\bz)^{k+1}-(\bw,\bz)^k}^2$
per iteration;
(iii) there exist subgradients with
$\norm{\bm v^{k+1}}\le(\eta\bar L+L_f)\,
\norm{(\bw,\bz)^{k+1}-(\bw,\bz)^k}$;
(iv) the whole sequence converges to a single critical point of $F$
with finite length; and
(v) if the KL exponent at the limit is $\theta=1/2$ --- which holds on
the RSC event of Theorem~\ref{thm:rsc} at stationary points with the
active-set structure of the penalty --- convergence is R-linear.
\end{theorem}

\begin{proof}
(i)--(iii) are Lemmas 1--3 of the algorithmic analysis: (i) follows from
local Lipschitzness of $\nabla f_H$ on compacts and geometric growth of
$L$; (ii) from adding the prox inequality (modulus $\eta L$) to the
accepted descent condition (constant $L$), which leaves
$\frac{(\eta-1)L}{2}\norm{\bDelta}^2$ (Remark~\ref{rem:eta}); (iii) from
first-order optimality of the prox step plus Lipschitz gradient
differences. Conditions (ii)--(iii) are (H1)--(H2) of the abstract
framework of Attouch--Bolte--Svaiter \cite{attouch2013convergence};
continuity (H3) holds along prox steps; definability gives the KL
property \cite{attouch2010proximal,bolte2014palm}; (iv) follows from the
framework and (v) from the \L{}ojasiewicz-exponent rate analysis
\cite{attouch2009convergence,li2018calculus}.
\end{proof}

\section{Statistical Recovery Theory}\label{sec:theory}

Throughout, $\alpha\coloneqq\ell_g^2/16$,
$\underline\alpha\coloneqq\alpha-\rho>0$, and
$r_H\coloneqq\dH/(8L_g)$ (the \emph{Huber-band radius}).

\subsection{Geometry}

\begin{lemma}[Superposition dictionary]\label{lem:incoh}
For any $\bDelta=(\bm\delta_w,\bm\delta_z)$,
$\norm{\bB\bDelta}_2^2\ge\norm{\bDelta}_2^2-
2\varepsilon\norm{\bm\delta_w}_1\norm{\bm\delta_z}_1$; on the cone
$\mathcal C(S,3)=\{\norm{\bDelta_{S^c}}_1\le3\norm{\bDelta_S}_1\}$,
$\norm{\bB\bDelta}_2^2\ge(1-16\varepsilon s)\norm{\bDelta}_2^2
\ge\tfrac12\norm{\bDelta}_2^2$.
\end{lemma}

\begin{proof}
$\norm{\bB\bDelta}^2=\norm{\bDelta}^2
+2\bm\delta_w^\top\bPhi^\top\bPsi\bm\delta_z$ and the cross term is
bounded entrywise by
$\varepsilon\norm{\bm\delta_w}_1\norm{\bm\delta_z}_1
\le\tfrac\varepsilon4\norm{\bDelta}_1^2\le4\varepsilon
s\norm{\bDelta}_2^2$ on the cone, where
$\norm{\bDelta}_1\le4\sqrt s\norm{\bDelta}_2$.
\end{proof}

\subsection{Restricted strong convexity}

\begin{theorem}[RSC of the Huberized loss]\label{thm:rsc}
Let Assumptions~\ref{as:sensing}--\ref{as:amenable} hold, $\dH\ge4\sigma$,
and $m\ge C\log(2n)$. With probability at least
$1-c_1e^{-c_2m}-2(2n)^{-2}-(mn)^{-1}$, for all $\bDelta$ with
$\norm{\bB\bDelta}_2\le r_H$,
\begin{equation}\label{eq:rsc}
\inner{\nabla f_H(\bbs+\bDelta)-\nabla f_H(\bbs)}{\bDelta}
\ge\alpha\norm{\bDelta}_2^2-\tau_0\tfrac{\log(2n)}{m}\norm{\bDelta}_1^2,
\end{equation}
with $\tau_0=c_3\bigl(L_g^2+\dH^2M_g^2/\ell_g^2\bigr)$.
\end{theorem}

\begin{proof}
Let $t_i=\inner{\ba_i}{\bB\bDelta}$, $u_i^\ast=\inner{\ba_i}{\bx^\ast}$,
$d_i=g(u_i^\ast+t_i)-g(u_i^\ast)$. The gradient-increment summand is
$[\phi_i(t_i)-\phi_i(0)]t_i$ with
$\phi_i(t)=-\psidH(e_i-(g(u_i^\ast+t)-g(u_i^\ast)))g'(u_i^\ast+t)$;
split $\phi_i(t_i)-\phi_i(0)=A_i+B_i$,
\[
A_i=[\psidH(e_i)-\psidH(e_i-d_i)]g'(u_i^\ast+t_i),
\]
\[
B_i=-\psidH(e_i)[g'(u_i^\ast+t_i)-g'(u_i^\ast)].
\]
\emph{Curvature part.} $\psidH$ nondecreasing, $g$ increasing, $g'>0$
give $A_it_i\ge0$ \emph{for every sample}. On
$\mathcal E_i=\{|e_i|\le\dH/2\}\cap\{|t_i|\le\dH/(2L_g)\}$, both $e_i$
and $e_i-d_i$ lie in the linear zone of $\psidH$, so
$A_it_i=g'(\xi_i)g'(u_i^\ast+t_i)t_i^2\ge\ell_g^2t_i^2$. Hence
$T_1\ge\frac{\ell_g^2}{m}\sum_ib_i\varphi_K(t_i)$ with
$b_i=\bm1\{|e_i|\le\dH/2\}$ ($\Prob(b_i{=}1)\ge\frac12$ by Chebyshev and
$\dH\ge4\sigma$), $K=\dH/(2L_g)$, and $\varphi_K$ the truncated
quadratic of Lemma~\ref{lem:A1}, which yields the uniform lower bound
$\frac{\ell_g^2}{8}\norm{\bB\bDelta}^2-cL_g^2\frac{\log2n}{m}
\norm{\bDelta}_1^2$ over $\norm{\bB\bDelta}\le K/4=r_H$.
\emph{Multiplier part.} $\E[B_it_i\mid\ba_i]=0$ (noise symmetry),
$|B_it_i|\le\dH M_gt_i^2$; Lemma~\ref{lem:A2} bounds
$\sup|{\textstyle\frac1m\sum}B_it_i|\le\frac{\ell_g^2}{32}
\norm{\bDelta}^2+c\frac{\dH^2M_g^2}{\ell_g^2}\frac{\log2n}{m}
\norm{\bDelta}_1^2$.
\emph{Assembly} with Lemma~\ref{lem:incoh} (cone case; the all-vector
case absorbs the incoherence deficit
$\le\frac\varepsilon2\norm{\bDelta}_1^2$ into the tolerance) gives
\eqref{eq:rsc}.
\end{proof}

\begin{remark}[The radius $r_H$ is intrinsic]
Huber loss is linear outside its band --- it has no curvature there; an
RSC radius proportional to $\dH$ is the correct statement, exactly as in
high-dimensional robust regression \cite{loh2017robust,sun2020adaptive}.
Self-consistency (final error $\le r_H$) holds for $m\gtrsim
(L_g^2\min(\sigma,\dH)/(\ell_g^2\dH))^2s\log(2n)$.
\end{remark}

\begin{lemma}[Gradient at the truth]\label{lem:gradtruth}
With probability $\ge1-2(2n)^{-2}-(mn)^{-1}$, for $m\ge\log(2n)$:
$\norm{\nabla f_H(\bbs)}_\infty\le
cL_g\min(\sigma,\dH)\sqrt{\log(2n)/m}$.
\end{lemma}

\begin{proof}
Coordinate $j$ is $\frac1m\sum_ie_i'g'(u_i^\ast)\inner{\ba_i}{\bm b_j}$
with $e_i'=\psidH(e_i)$ mean-zero, bounded by $\dH$, variance
$\le\min(\sigma^2,\dH^2)$. Condition on $\bA$; the conditional
sub-Gaussian parameter is
$\frac{L_g^2\min(\sigma^2,\dH^2)}{m}\cdot\frac1m\norm{\bA\bm b_j}^2$,
and $\frac1m\norm{\bA\bm b_j}^2\le2$ w.p.\ $1-e^{-cm}$ ($\chi_m^2$
Bernstein). Union over $2n$ columns.
\end{proof}

\subsection{Main error bound}

The estimator carries the $\ell_1$ side constraint
$\Omega=\{\norm{\bb}_1\le R_1\}$, $R_1\ge2\norm{\bbs}_1$
\cite{loh2015regularized}; Remark~\ref{rem:interior} shows it is
inactive at the solution.

\begin{theorem}[Error at stationary points]\label{thm:main}
On the RSC event \eqref{eq:rsc}, with
\begin{equation}\label{eq:lamchoice}
\lambda\ge\max\Bigl\{4\norm{\nabla f_H(\bbs)}_\infty,\;
\tfrac{8\tau_0R_1\log(2n)}{m}\Bigr\},
\end{equation}
every stationary point $\bbt$ of $F$ over $\Omega$ with
$\norm{\bbt-\bbs}_2\le r_H$ satisfies
\[
\norm{\bbt-\bbs}_2\le\frac{3\lambda\sqrt s}{2\underline\alpha},
\qquad
\norm{\bbt-\bbs}_1\le\frac{6\lambda s}{\underline\alpha}.
\]
\end{theorem}

\begin{proof}
Write $\bDelta=\bbt-\bbs$, $a=\norm{\bDelta_S}_1$,
$b=\norm{\bDelta_{S^c}}_1$.
(1) Constrained stationarity against the feasible $\bbs$:
$\inner{\nabla f_H(\bbt)-\nabla f_H(\bbs)}{\bDelta}\le
-\inner{\bm\xi}{\bDelta}-\inner{\nabla f_H(\bbs)}{\bDelta}$,
$\bm\xi\in\partial P(\bbt)$.
(2) Weak convexity of $P+\frac\rho2\norm{\cdot}^2$:
$-\inner{\bm\xi}{\bDelta}\le P(\bbs)-P(\bbt)
+\frac\rho2\norm{\bDelta}^2$.
(3) Penalty comparison via (P1)--(P4) and
Assumption~\ref{as:amenable}:
$P(\bbs)-P(\bbt)\le\lambda a-\lambda b
+\frac\rho2\norm{\bDelta_{S^c}}^2$.
(4) Gradient term: $\le\frac\lambda4(a+b)$ by \eqref{eq:lamchoice}.
(5) Tolerance: $\norm{\bDelta}_1\le2R_1$ on $\Omega$, so
$\tau_0\frac{\log2n}{m}\norm{\bDelta}_1^2\le\frac\lambda4(a+b)$ by the
second branch of \eqref{eq:lamchoice}.
(6) Assembling with \eqref{eq:rsc} and collecting $\rho$-terms:
$\underline\alpha\norm{\bDelta}^2\le\frac{3\lambda}2a-\frac\lambda2b$.
(7) Nonnegativity of the left side forces $b\le3a$: the cone, whence
$\norm{\bDelta}_1\le4\sqrt s\norm{\bDelta}_2$.
(8) Dropping $-\frac\lambda2b$ and using $a\le\sqrt s\norm{\bDelta}_2$:
$\norm{\bDelta}_2\le\frac{3\lambda\sqrt s}{2\underline\alpha}$, and the
$\ell_1$ bound follows.
\end{proof}

\begin{remark}[Side constraint inactive]\label{rem:interior}
The conclusion gives $\norm{\bbt}_1\le\norm{\bbs}_1
+6\lambda s/\underline\alpha<R_1$ for $\lambda$ of the prescribed order
and $m\gtrsim s\log2n$: constrained and unconstrained stationary points
coincide, and NLD-PALM needs no $\ell_1$ projection (verified
numerically in all experiments).
\end{remark}

\begin{corollary}[Rates]\label{cor:rate}
With $\lambda\asymp L_g\min(\sigma,\dH)\sqrt{\log(2n)/m}$ and
$m\gtrsim\max\{s\log(2n/s),\,(\tau_0/\underline\alpha)^2s\log
2n,\,(\tau_0\norm{\bbs}_2/\sigma)^2s\log(2n)\}$:
\[
\norm{\bbt-\bbs}_2\lesssim\frac{L_g}{\underline\alpha}\,
\sigma\sqrt{\frac{s\log(2n)}{m}},
\]
under sub-Gaussian \emph{or} finite-variance symmetric noise.
\end{corollary}

\subsection{Oracle property}

\begin{proposition}[Oracle property]\label{prop:oracle}
Let $E_2=\frac{3\lambda\sqrt s}{2\underline\alpha}$ and assume
$\min_{j\in S}|\beta_j^\ast|\ge a\lambda+E_2$.
(a) If $\supp(\bbt)\subseteq S$, then $\bbt$ solves the unpenalized
problem on $\R^S$ and
$\norm{\bbt-\bbs}_2\le\frac2\alpha\norm{\nabla_Sf_H(\bbs)}_2
\lesssim\frac{L_g}\alpha\min(\sigma,\dH)\sqrt{s/m}$ w.h.p.\ --- no
$\log n$, no $\lambda$-bias.
(b) If additionally $C_LE_2\le\lambda/4$ with $C_L=c(L_g^2+\dH M_g)$
and $E_2<\gamma\lambda/2$, then $\supp(\bbt)\subseteq S$.
\end{proposition}

\begin{proof}
(a) Beta-min and $\norm{\bDelta}_\infty\le E_2$ put all $j\in S$ in the
zero-derivative region of $P_\lambda$ (P4), so stationarity reduces to
$\nabla_Sf_H(\bbt)=0$; RSC on the fixed $s$-dimensional subspace and a
vector-Bernstein bound on $\norm{\nabla_Sf_H(\bbs)}_2$ (no union over
$2n$ coordinates --- this is where $\log n$ disappears) give the rate.
(b) For $j\in S^c$ with $\bbt_j\ne0$, MCP stationarity forces
$|\nabla_jf_H(\bbt)|=\lambda-|\bbt_j|/\gamma$ or
$|\bbt_j|\ge\gamma\lambda$. The coordinatewise Lipschitz bound
$\norm{\nabla f_H(\bbt)-\nabla f_H(\bbs)}_\infty\le C_L\norm{\bDelta}_2$
(Cauchy--Schwarz with upper restricted eigenvalues) gives
$|\nabla_jf_H(\bbt)|\le\lambda/2$, forcing
$|\bbt_j|\ge\gamma\lambda/2>E_2\ge|\bbt_j|$ --- contradiction. The
condition $C_LE_2\le\lambda/4$ amounts to $\sqrt
s\lesssim\underline\alpha/C_L$; the primal--dual witness of
\cite{loh2017support} removes this restriction and is left as an
optional appendix.
\end{proof}

\begin{remark}[$\ell_1$ cannot debias]
$\ell_1$ stationary points satisfy $|\nabla_jf_H|=\lambda$ at every
nonzero coordinate, embedding a $c\lambda\sqrt s$ bias for any $\lambda$
compatible with \eqref{eq:lamchoice} --- visible as the error plateau in
Fig.~\ref{fig:main}.
\end{remark}

\subsection{Unknown link}\label{ssec:unknown}

Assume $g$ unknown, monotone, with $\mu_g=\E[g'(\gamma)]\ne0$,
$\theta_g^2=\E[(g(\gamma)-\mu_g\gamma)^2]$, $\gamma\sim\mathcal N(0,1)$,
$\norm{\bx^\ast}_2=1$; let $\nu^2=\theta_g^2+\sigma^2$. Run the same
program with the \emph{linear} Huberized loss and adaptive clipping
$\dH=2\nu\sqrt{\log m}$.

\begin{lemma}[Clipped decoupling]\label{lem:decouple}
With probability $\ge1-c(2n)^{-2}$,
$\norm{\frac1m\sum_i\psidH(\xi_i)\bB^\top\ba_i}_\infty\le
c\,\nu\sqrt{\log(2n)\log(m)/m}$, where
$\xi_i=g(u_i^\ast)-\mu_gu_i^\ast+e_i$.
\end{lemma}

\begin{proof}
Split $\ba_i=\gamma_i\bx^\ast+\bm P_\perp\ba_i$ with
$\gamma_i=\inner{\ba_i}{\bx^\ast}$; $\xi_i$ depends on
$(\gamma_i,e_i)$ only. \emph{Orthogonal part:} conditionally Gaussian
with variance $\frac1m\cdot\frac1m\sum\psidH(\xi_i)^2\le\frac{c\nu^2}m$
w.h.p.\ (Bernstein), giving $c\nu\sqrt{\log(2n)/m}$ after a union bound.
\emph{Parallel bias:} $\E[\xi\gamma]=0$ by Stein's identity
($\E[g(\gamma)\gamma]=\E[g'(\gamma)]=\mu_g$), and
$|\E[\psidH(\xi)\gamma]|\le\sqrt{\E[\xi^2\bm1\{|\xi|>\dH\}]}
\le c\nu e^{-\dH^2/(c'\nu^2)}\le c\nu m^{-4/c'}$ by sub-Gaussianity of
$\xi$ and the choice of $\dH$.
\emph{Parallel deviation:} $\psidH(\xi_i)\gamma_i$ has envelope
$\dH|\gamma_i|$ and variance $\le c\nu^2\log m$; Bernstein plus the
union bound gives the dominant term
$c\nu\sqrt{\log m}\sqrt{\log(2n)/m}$.
\end{proof}

\begin{theorem}[Unknown-link recovery]\label{thm:unknown}
Under Assumptions~\ref{as:sensing}, \ref{as:incoh},
\ref{as:noise}(a), \ref{as:amenable}, with $m\ge Cs\log(2n)$ and
$\lambda\asymp\nu\sqrt{\log(2n)\log(m)/m}$: with probability
$\ge1-c_1e^{-c_2m}-c_3n^{-1}$, every stationary point of the
linear-surrogate objective with $\norm{\bbt-\mu_g\bbs}_2\le r_H$
satisfies
\[
\norm{\bbt-\mu_g\bbs}_2\lesssim\frac{1}{1/16-\rho}\,
\nu\,\sqrt{\frac{s\log(2n)\log m}{m}}.
\]
\end{theorem}

\begin{proof}
Write $y_i=\mu_g\inner{\ba_i}{\bx^\ast}+\xi_i$. The loss is
linear-Huber: Theorem~\ref{thm:rsc} applies with $g=\mathrm{id}$
($\ell_g=L_g=1$, $M_g=0$; the multiplier part vanishes identically,
so no noise-symmetry is needed for RSC). Lemma~\ref{lem:decouple}
replaces Lemma~\ref{lem:gradtruth}. Theorem~\ref{thm:main}'s
deterministic core applies verbatim around $\mu_g\bbs$.
\end{proof}

\section{Experiments}\label{sec:experiments}

Protocol: $n=512$, $s_1=s_2=16$, $\bPhi=\bm I$, $\bPsi=$ DCT,
$g=\tanh$, $\sigma=0.05$, $\dH=4\sigma$. Regularization uses the
\emph{frozen} rule $\lambda=c\,\sigma\sqrt{\log(2n)/m}$, with $c$
calibrated once on a $5$-seed pilot at $m=400$ and never re-tuned
($c=8$ for $\ell_1$, $16$ for SCAD, $24$ for MCP, $4$ for
$\ell_{1/2}$); the DHT baseline \cite{soltani2017fast} step is frozen at
its pilot-best likewise, and receives the true $(s_1,s_2)$.
Nonconvex methods are warm-started from the $\ell_1$ solution. $15$
Monte-Carlo trials per point, deterministic seeds, code provided.

\begin{figure*}[t]
\centering
\includegraphics[width=0.98\textwidth]{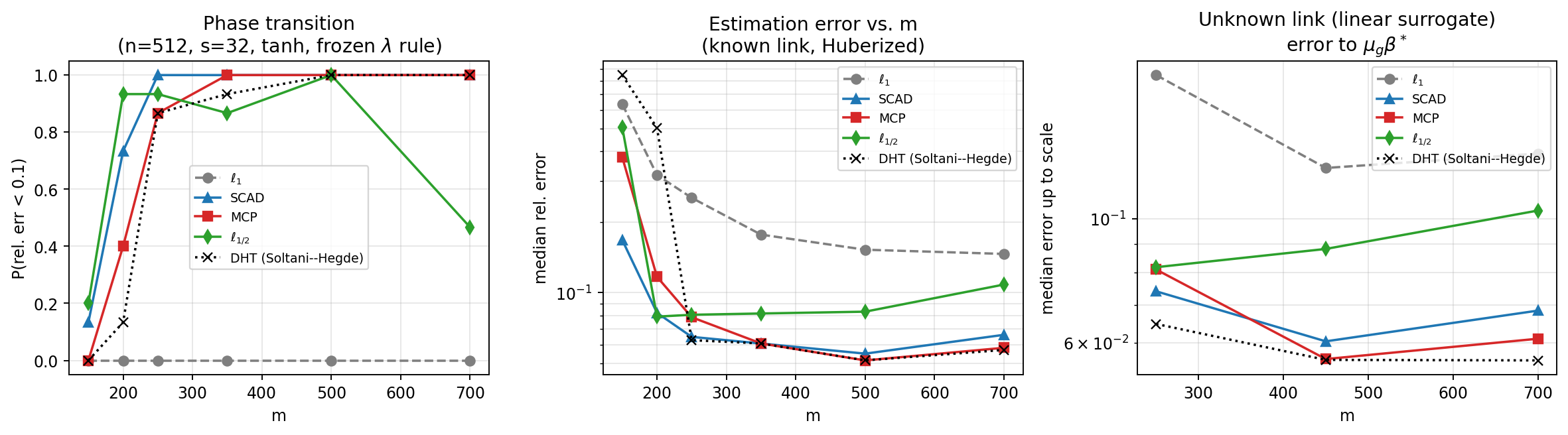}
\caption{Left, center: phase transition and median error vs.\ $m$
(known link, Huberized). SCAD crosses at $m\approx200$--$250$ vs.\
$m\approx250$--$350$ for DHT, while $\ell_1$ plateaus at its shrinkage
bias $\approx0.15$. The frozen $\sqrt{\log(n)/m}$ rule mismatches
$\ell_{1/2}$ half-thresholding at large $m$, consistent with its
exclusion from the statistical theory. Right: unknown-link estimator;
the $\theta_g$ floor of Theorem~\ref{thm:unknown} appears beyond
$m\approx450$.}
\label{fig:main}
\end{figure*}

\begin{figure}[t]
\centering
\includegraphics[width=\columnwidth]{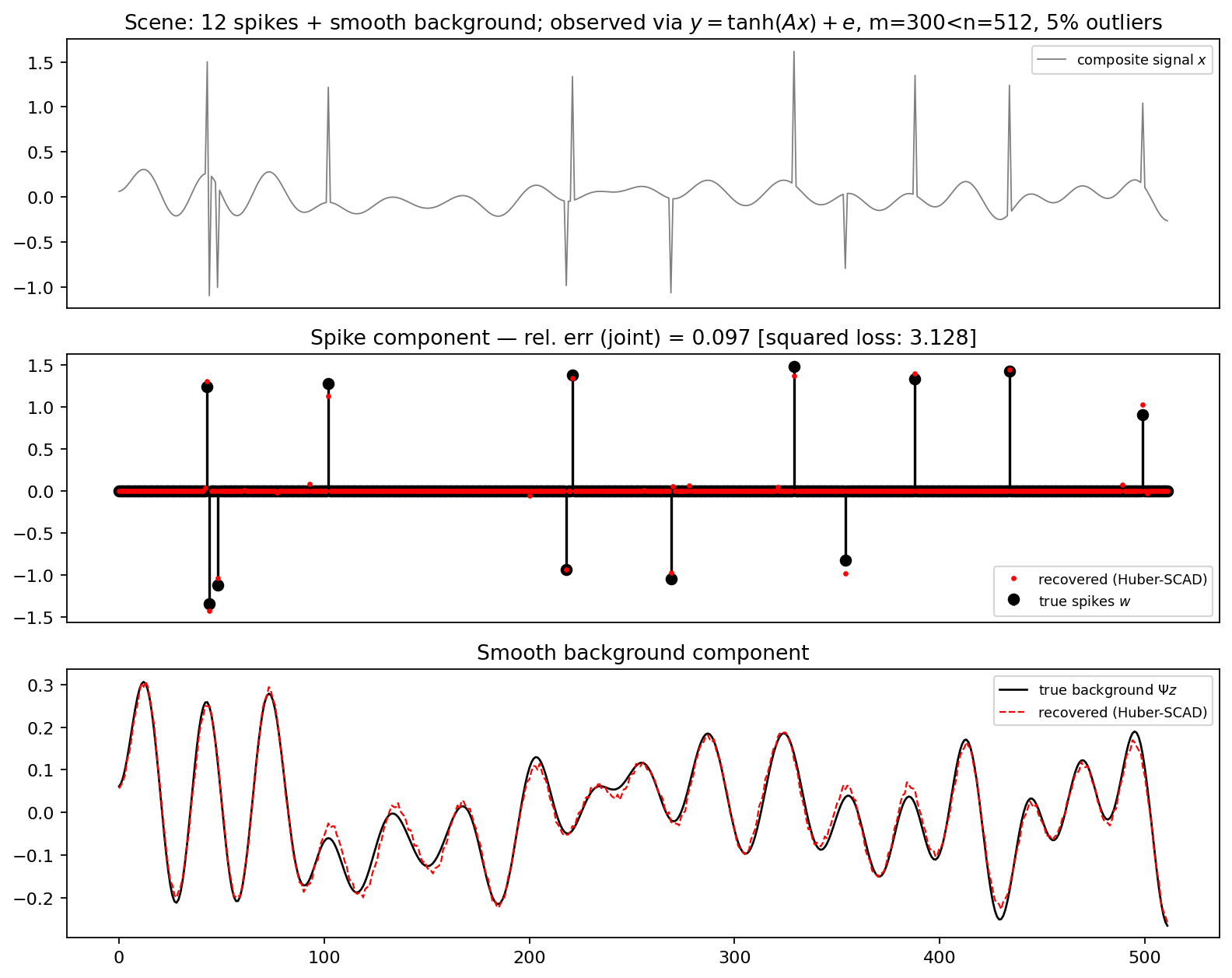}
\caption{Saturation demixing: $12$ spikes + smooth DCT background,
$y=\tanh(\bA\bx)+\bm e$, $m=300<n=512$, $5\%$ gross outliers.
Huberized SCAD: relative error $0.097$; squared-loss SCAD: $3.13$.}
\label{fig:demo}
\end{figure}

\textbf{Phase transition.} Fig.~\ref{fig:main}: under the frozen rule
SCAD/MCP transition $\approx1.3$--$1.4\times$ earlier than DHT; in the
scarce regime ($m\le200$) the gap is decisive (SCAD $0.08$--$0.17$ vs.\
DHT $0.50$--$0.85$). Under per-instance oracle tuning the gap widens to
$\approx1.5$--$2\times$; we report the frozen protocol as the
practically meaningful one. $\ell_1$ never attains exact recovery,
matching Proposition~\ref{prop:oracle}'s bias dichotomy.

\textbf{Robustness.} With $5\%$ gross outliers ($\pm3$): Huberized SCAD
$0.073$ vs.\ squared-loss SCAD $2.54$ median relative error
($35\times$); Student-$t(2)$ noise: $0.12$ vs.\ $0.54$; clean Gaussian
noise: parity ($0.060$ vs.\ $0.060$). Huberization is free insurance.
An honest negative: at high \emph{Gaussian} noise ($\sigma\ge0.1$ for
$\tanh$) recovery degrades for \emph{all} losses --- an
information-theoretic saturation effect, not loss fragility.

\textbf{Scaling checks.} Measured error scales linearly in $\sigma$
(slope $1.0$ on a log-log grid up to the saturation regime) and as
$m^{-1/2}$ over the pre-floor range, matching
Corollary~\ref{cor:rate}.

\textbf{Application.} Fig.~\ref{fig:demo}: spikes and a smooth
background demixed through a saturating amplifier with outliers; all
spikes located at $m/n\approx0.59$.

\section{Conclusion}\label{sec:conclusion}

We provided the first regularization-based treatment of nonlinear
sparse demixing: a Huberized folded-concave estimator, a provably
convergent algorithm whose over-relaxation requirement we identified,
RSC-based guarantees holding at every localized stationary point in
both known- and unknown-link regimes, oracle behavior unavailable to
convex demixing, and robustness with only finite-variance symmetric
noise. Future work: a primal--dual witness removing the dimensional
condition in Proposition~\ref{prop:oracle}(b), statistical theory for
$\ell_q$, periodic and quantized links, and a deep-unfolded NLD-PALM
with trainable threshold functions.

\appendix

\section{Proximal operators}\label{app:prox}
With step $\beta=1/(\eta L)$: \textbf{MCP} ($\gamma>\beta$):
$0$ if $|u|\le\lambda\beta$;
$\mathrm{sign}(u)(|u|-\lambda\beta)/(1-\beta/\gamma)$ if
$\lambda\beta<|u|\le\gamma\lambda$; $u$ otherwise.
\textbf{SCAD} ($a>1+\beta$): soft threshold at $\lambda\beta$ for
$|u|\le\lambda(1+\beta)$;
$((a-1)u-\mathrm{sign}(u)a\lambda\beta)/(a-1-\beta)$ for
$|u|\le a\lambda$; $u$ otherwise.
\textbf{$\ell_{1/2}$}: half-thresholding \cite{xu2012half} with
effective parameter $2\lambda\beta$ (their quadratic term carries no
$1/2$). All three are unit-tested against brute-force minimization of
the scalar prox objective; the test suite accompanies the code.

\section{Uniform bounds}\label{app:uniform}

Let $\varphi_K(t)=t^2$ on $|t|\le K/2$, $(K-|t|)^2$ on
$K/2<|t|\le K$, $0$ beyond: $K$-Lipschitz,
$t^2\bm1\{|t|\le K/2\}\le\varphi_K\le t^2\bm1\{|t|\le K\}$,
self-similar ($\varphi_K(t)=\omega^2\varphi_{K/\omega}(t/\omega)$),
monotone in $K$.

\begin{lemma}[Uniform truncated-quadratic lower bound]\label{lem:A1}
Let $b_i\in\{0,1\}$ iid, $\Prob(b_i{=}1)=p\ge\frac12$, independent of
$\{\ba_i\}$, and $m\ge c\log(2n)$. With probability
$\ge1-c_1e^{-c_2m}$, for all $\bDelta$ with
$\norm{\bm v}_2\le K/4$ ($\bm v=\bB\bDelta$):
\[
\frac1m\sum_ib_i\varphi_K(\inner{\ba_i}{\bm v})\ge
\frac p4\norm{\bm v}_2^2-c\,\frac{\log(2n)}m\norm{\bDelta}_1^2 .
\]
\end{lemma}

\begin{proof}
Fix a scale $\omega=\norm{\bm v}_2$ and use the per-scale level
$K_\omega=4\omega\le K$ (monotonicity). \emph{Population:}
$\E[b_i\varphi_{K_\omega}]\ge p\,\omega^2\E[\gamma^2\bm1\{|\gamma|\le2\}]
\ge0.73\,p\,\omega^2$. \emph{Deviation:} by self-similarity, work with
$\varphi_4$ ($4$-Lipschitz) on normalized variables; symmetrization and
Ledoux--Talagrand contraction \cite{ledoux1991probability} give
$\E\sup Z\le16\,\omega\,T\sqrt{2\log(4n)/m}$ over
$\{\norm{\bDelta}_1\le T\}$, since each coordinate of
$\frac1m\sum_i\varepsilon_i\bB^\top\ba_i$ is $\mathcal N(0,1/m)$;
Talagrand concentration (envelope $4\omega^2$) upgrades to probability
$1-e^{-c_2m}$. \emph{Absorb or peel:} if
$32T\sqrt{2\log(4n)/m}\le\frac p4\omega$ the bound is direct; otherwise
$\omega^2\lesssim\frac{\log2n}mT^2$ and the tolerance term alone
validates the display. Dyadic peeling over
$\omega\in(2^{-j-1}K/4,2^{-j}K/4]$ with a union over
$O(\log m)$ shells completes the proof.
\end{proof}

\begin{lemma}[Mean-zero multiplier bound]\label{lem:A2}
In the setting of Theorem~\ref{thm:rsc}, with probability
$\ge1-(2n)^{-2}$, uniformly over $\norm{\bm v}_2\le r_H$:
\[
\Bigl|\frac1m\sum_i\psidH(e_i)h_{u_i^\ast}(t_i)\Bigr|
\le\frac{\ell_g^2}{32}\norm{\bDelta}_2^2
+c\,\frac{\dH^2M_g^2}{\ell_g^2}\frac{\log(2n)}m\norm{\bDelta}_1^2,
\]
where $h_u(t)=[g'(u+t)-g'(u)]\,t$.
\end{lemma}

\begin{proof}
$h_u$ is $2M_gK_\omega$-Lipschitz with $h_u(0)=0$ on the per-scale band;
$\psidH(e_i)$ are iid mean-zero (symmetry), bounded by $\dH$,
independent of $\{\ba_i\}$. The multiplier inequality and contraction
\cite{ledoux1991probability} reduce the supremum to the linear class,
giving per shell $\lesssim\dH M_g\omega\norm{\bDelta}_1
\sqrt{\log(2n)/m}$; Young's inequality
($ab\le\eta a^2+b^2/4\eta$, $\eta=\ell_g^2/32$) splits this into the
displayed form; Talagrand concentration and the shell/coordinate union
complete the proof.
\end{proof}

\end{document}